\documentclass{article}

\usepackage{PRIMEarxiv}
\usepackage[T1]{fontenc}    
\usepackage{hyperref}       
\usepackage{url}            
\usepackage{booktabs}       
\usepackage{amsfonts}       
\usepackage{nicefrac}       
\usepackage{microtype}      
\usepackage{lipsum}
\usepackage{fancyhdr}       
\usepackage{graphicx}       
\graphicspath{{media/}}     

\usepackage{algorithm} 

\usepackage{amsfonts}
\usepackage[noend]{algpseudocode} 
\usepackage{mwe}

\usepackage{multirow}
\usepackage{array}
\usepackage{dsfont}
\usepackage{amsmath}
\usepackage{amssymb}
\usepackage{comment}
\DeclareMathOperator*{\argmax}{arg\,max}

\newcolumntype{P}[1]{>{\centering\arraybackslash}p{#1}}

\usepackage{graphicx}
\usepackage{caption}
  
\usepackage{subcaption}
\usepackage[utf8x]{inputenc}
\usepackage{graphicx,xcolor} 
\usepackage[framemethod=tikz]{mdframed}
\usepackage[normalem]{ulem}
\useunder{\uline}{\ul}{}

\usepackage[english]{babel}
\usepackage{amssymb}
\usepackage{amsfonts} 

\usepackage[english]{babel}

\usepackage[T1]{fontenc}
\usepackage[font=scriptsize]{caption}

\usepackage{amsthm}

\theoremstyle{definition}
\newtheorem{definition}{Definition}[]
\theoremstyle{plain}
\newtheorem{theorem}{Theorem}

\theoremstyle{remark}
\newtheorem*{remark}{Remark}

\title{Modeling Human Driver Interactions \\ Using an Infinite Policy Space Through Gaussian Processes
}

\author{
  Cem Okan Yaldiz \\
  School of Electrical and Computer Engineering \\
  Georgia Institute of Technology \\
  Atlanta\\
  \texttt{cyaldiz3@gatech.edu} \\
   \And
  Yildiray Yildiz \\
  Department of Mechanical Engineering \\
  Bilkent University \\
  Ankara\\
  \texttt{yyildiz@bilkent.edu.tr} \\
}

\begin{document}
\maketitle

\begin{abstract}
This paper proposes a method for modeling human driver interactions that relies on multi-output gaussian processes. The proposed method is developed as a refinement of the game theoretical hierarchical reasoning approach called "level-\(k\) reasoning" which conventionally assigns discrete levels of behaviors to agents. Although it is shown to be an effective modeling tool, the level-\(k\) reasoning approach may pose undesired constraints for predicting human decision making due to a limited number (usually 2 or 3) of driver policies it extracts. The proposed approach is put forward to fill this gap in the literature by introducing a continuous domain framework that enables an infinite policy space. By using the approach presented in this paper, more accurate driver models can be obtained, which can then be employed for creating high fidelity simulation platforms for the validation of autonomous vehicle control algorithms. The proposed method is validated on a real traffic dataset and compared with the conventional level-\(k\) approach to demonstrate its contributions and implications.
\end{abstract}


\section{Introduction}
One of the promises of the autonomous vehicle (AV) technology is to decrease the crash rates related to human faults, which will lead to a safer transportation. It is argued that AVs will help not only in preventing accidents but also in solving traffic congestions and in reducing emissions \cite{yurtsever2020survey}. Despite these promises, however, the emergence of AVs in regular commercial use is still pending mainly because of safety issues. Indeed, safety is the very first challenge for robotic systems where humans are in the loop \cite{sand2013closing}. Therefore, their safety must be validated before AVs become a usual sight at everyday traffic. 

There is a high probability that AVs will have their own roads/lanes that are dedicated to their use only \cite{carvalho2015automated}. However, in a scenario where human drivers and AVs coexist in the same environment, the safety problem becomes a tough challenge due to the very nature of hard-to-predict human behavior. In \cite{kalra2016driving}, it is stated that to prove AVs' safety, hundreds of millions of miles are needed to be driven without fatality, which correspond to years of observation. This particular result makes it difficult to validate AV technology in real roads. Also, observing AVs in different test cases that can occur in real life is non-trivial. On the other hand, using computer simulations for testing both decreases the required time and effort significantly and help create complex environments for observing AVs' reactions. Therefore, creating reliable simulation environments that can mimic the real world has utmost importance for the validation of AV control algorithms. To achieve this, modeling human driver interactions as accurate as possible is necessary.

Joint employment of reinforcement learning and game theory is demonstrated to be an effective tool for modeling human interactions. Examples for the utilization of this method to predict human-human and human-automation interactions can be found in \cite{yildiz2014predicting}, \cite{musavi2017unmanned}, \cite{albaba2019modeling} and \cite{albaba20213D} for aviation, and in \cite{li2017game}, \cite{koprulu2021act}, \cite{karimi2020receding}, \cite{zhang2017finite} and \cite{jain2018multi} for automotive applications. The main game theoretical approach used in these studies are called the "level-\(k\) reasoning" \cite{camerer2004cognitive}, \cite{stahl1995players}, \cite{crawford2007level}, \cite{costa2009comparing}. This method is based on the idea that humans have different levels of decision making algorithms, where level-\(k\) is a best response to level-(\(k\)-1). Although level-\(k\) reasoning is shown to have a reasonable match with real human driving behavior \cite{albaba2020driver}, it suffers from the limited number of driver policies. In this paper, we propose a solution to this problem by introducing an infinite policy space for human interactions. Our solution uses the limited policies suggested by the conventional level-\(k\) approach as inputs to a multi-output Gaussian Process (GP) to obtain real-valued driver levels. Furthermore, we introduce and prove a "best response" theorem which precisely defines the hierarchy between real-valued reasoning levels. Therefore, the proposed solution can be considered as a level-\(k\) refinement. We also validate our results using a real traffic dataset. A preliminary version of this work is submitted to 2022 American Control Conference.

This paper is organized as follows: In Section \(\ref{relatedwork}\), related studies in the literature on human driver modeling through different approaches and on level-\(k\) game theoretical reasoning concept are given. In Section \(\ref{background}\), we briefly summarize the minimum background knowledge about single-output and multi-output GPs along with the fundamental premises of level-\(k\) framework. The method developed in this paper to transfer the problem from integer-valued levels to real-valued levels is described in Section \(\ref{method}\). Furthermore, a best response theorem for the interaction of agents with real-valued levels is introduced in Section \(\ref{method}\). In Section \(\ref{validation}\), we present the construction of gaussian processes for the task and how we validate our approach on real traffic data. The results that demonstrate contributions of the presented approach are shared in Section \(\ref{results}\) along with a discussion. Finally, an overall summary is provided in Section \(\ref{conclusion}\).

\section{Related Work}
\label{relatedwork}
There is a considerable body of work in human driver modeling in terms of state estimation, trait estimation, intention estimation, motion prediction or anomaly detection \cite{brown2020modeling}. The proposed approach in this paper is related to the context of \emph{trait estimation} where the style of the driver is attempted to be extracted. Several solutions for these problems are introduced in the literature. Among these, Naive Bayes models \cite{gillmeier2019prediction}, gaussian mixture models \cite{angkititrakul2009evaluation}, markov chain models \cite{aasljung2019probabilistic}, markov state space models \cite{lefevre2013intention}, optimal control based models \cite{han2019human}, inverse reinforcement learning based models \cite{sadigh2016planning}, \cite{sun2018probabilistic}, Gaussian Process based models \cite{armand2013modelling}, \cite{tran2014online}, \cite{kruger2019probabilistic}, \cite{guo2019modeling}, game theoretical approaches \cite{schester2019longitudinal}, \cite{kang2018modeling}, \cite{yan2018game}, \cite{talebpour2015modeling}, Gated Recurrent Unit based models \cite{wang2017capturing}, joint use of recurrent and convolutional neural networks \cite{gebert2019end}, joint use of recurrent and generative networks \cite{kuefler2017imitating}, and joint use of generative networks and imitation learning \cite{ho2016generative} can be counted.

In recent years, a new method combining level-\(k\) game theoretical approach with reinforcement learning is proposed and used in several modeling domains \cite{yildiz2014predicting}, \cite{musavi2017unmanned}, \cite{musavi2016game}, \cite{albaba2019stochastic}, \cite{tian2018adaptive}, \cite{li2019game}, \cite{garzon2019game}. This method distinguishes itself from others by its ability to model several (>100) agents as strategic decision makers, simultaneously, instead of modeling a single decision maker and assigning predetermined behavioral patterns to the rest of the agents. In all of these studies, human drivers are modeled via discrete levels. What distinguishes the method presented in this paper from the earlier ones is that we abandon the belief that the humans can be successfully modeled via only  discrete levels. We claim that humans can have behavioral models in a continuous range within the intervals between these discrete levels. We introduce and prove a theorem about the hierarchy of real-valued reasoning levels, which establishes the proposed method as a \emph{level-\(k\) refinement}. In brief, our contributions in this paper are as follows: 
\begin{enumerate}
  \item We put forward a refinement on level-\(k\) reasoning concept that allows us to work between discrete levels in continuous domain for modeling human behavior. To our knowledge, this is the first study on continuous level-\(k\) modeling.
  \item We develop the necessary theoretical background to provide a hierarchical structure among real-valued reasoning levels.
  \item Newly created continuous driver models are investigated on a real dataset to provide statistical analysis for model prediction capability. The results are compared with the earlier results using benchmark cases.
\end{enumerate}

\section{Background} 

\label{background}

In this section, we provide the minimum necessary background before we introduce the infinite dimensional policy space in the following sections.   

\subsection{Level-k Reasoning, Deep Reinforcement Learning and Their Synergistic Employment }
The concept of level-\(k\) reasoning holds that humans have different levels of reasoning in decision making. In this framework, a level-0 agent do not take into account other agents' possible actions when he/she determines his/her own actions. In other words, level-0 represents a "non-strategic" agent and only follows predetermined rules. On the other hand, a level-1 agent treats other agents in the scenario as if they are level-0 and acts by maximizing rewards according to this assumption. Similarly, a level-2 agent treats other agents as if they are level-1 and moves accordingly. Hierarchically, a level-\(k\) agent chooses its actions by maximizing its rewards based on the assumption that other agents in the environment are following a level-(\(k\)-1) reasoning.

In order to obtain the best responses to level-(\(k\)-1) drivers, we use reinforcement learning (RL) to maximize the level-\(k\) agent's utility in the time extended scenario. Since we have a large state space, instead of a tabular RL approach, we prefer to use the deep Q-network (DQN). In-depth descriptions of level-\(k\) reasoning, DQN and their conjoint usage for modeling can be found at \cite{lee2012game}, \cite{mnih2015human} and \cite{albaba2020driver}, respectively.

\subsection{Gaussian Processes}
In this section, we briefly explain two types of Gaussian Processes (GP), namely, the single output and multi-output GPs.
\subsubsection{Single Output Gaussian Processes}
A GP is a random process for which the joint distribution of any finite subset of random variables follows a multi-variate Gaussian distribution. A function \(f : \mathbb{R} \rightarrow \mathbb{R}\) that is modeled by a GP is expressed as  
\begin{equation} \label{eq16}
\begin{split}
f(x) \sim \mathcal{GP}(\mu(x) , k(x, x')),
\end{split}
\end{equation}
where \(\mu(x)\) is the mean function and \(k(x, x')\) is the covariance function (kernel) which can be explicitly written as
\begin{equation}
    \mu(x) = E[f(x)],
\end{equation}
\begin{equation}
    k(x, x') = cov(f(x), f(x')) = E[(f(x) - \mu(x))(f(x')-\mu(x'))],
\end{equation}
where \(E[\hspace{0.1cm}]\) and \(cov()\) represent the expectation and covariance, respectively.
From the definition of GP, any \(n\) samples chosen from the process will also be jointly Gaussian, which can be stated as 
\begin{equation}
    \begin{bmatrix}
        f(x_1)\\
        . \\
        . \\
        . \\
        f(x_n)
    \end{bmatrix}
    \sim 
    \mathcal{N}(
    \begin{bmatrix}
        \mu(x_1)\\
        . \\
        . \\
        . \\
        \mu(x_n)
    \end{bmatrix},
    \Sigma)
\end{equation}
where \( \Sigma \in \mathbb{R}^{n \times n}\) is the covariance matrix containing \(k(x_i, x_j) \hspace{0.1cm}i, j=1, 2, ... , n, \) at the \(i^{th}\) row and \(j^{th}\) column. A valid \(\Sigma\) is positive definite, symmetric and invertible.

Given a GP, which is created by using observed samples \(\mathcal{D}=\{(x_i, f(x_i))\}\), \(i = 1,...,n\) with the mean \(\mu(x)\) and the kernel \(k(x, x')\), we want to make predictions of \(f(x_j^*)\) at points \(x_j^*\), \(j = 1, 2, ..., m\). The output vectors \(\boldsymbol{f} = [f(x_1), ..., f(x_n)]^T\) and \(\boldsymbol{f^*} = [f(x_1^*), ..., f(x_m^*)]^T\) are jointly Gaussian and their distribution can be expressed as

\begin{equation}
    \begin{bmatrix}
    \boldsymbol{f}\\
    \boldsymbol{f^*}
 \end{bmatrix}
 \sim \mathcal{N} (\boldsymbol{\mu}, \boldsymbol{\Sigma}),
\end{equation}
where 
\begin{equation}
\label{eqn:covsubmatrix}
    \boldsymbol{\Sigma} = \begin{bmatrix}
    \Sigma & \Sigma_*^T\\
    \Sigma_* & \Sigma_{**}\end{bmatrix}, \boldsymbol{\mu} = \begin{bmatrix}
    \mu(x) \\
    \mu(x^*)
\end{bmatrix}.
\end{equation}
The covariance sub-matrices \(\Sigma \in \mathbb{R}^{n \times n}\), \(\Sigma_* \in \mathbb{R}^{m \times n}\) and \(\Sigma_{**} \in \mathbb{R}^{m \times m}\) contain \(k(x_i, x_j)\), \(k(x_i^*, x_j)\) and \(k(x_i^*, x_j^*)\), respectively, as elements. 

We can predict the distribution of \(\boldsymbol{f^*}\) using the conditional probability, given \(\mathcal{D}\). Since the underlying assumption is a Gaussian distribution, the conditional distribution of \(\boldsymbol{f^*}\) will also be Gaussian which can be represented as
\begin{equation} 
\label{eqn:predictivemean1}
\begin{split}
\boldsymbol{f^*} | x, \boldsymbol{f}, x^* \sim \mathcal{N}(\mu_{\boldsymbol{f^*}|\mathcal{D}}, \Sigma_{\boldsymbol{f^*}|\mathcal{D}}).
\end{split}
\end{equation}
where \(\mu_{\boldsymbol{f^*}|\mathcal{D}} = \mu(x^*) + \Sigma_*\Sigma^{-1}\boldsymbol{f}\) and \(\Sigma_{\boldsymbol{f^*}|\mathcal{D}} = \Sigma_{**} - \Sigma_*\Sigma^{-1}\Sigma_*^T\).
Setting the means \(\mu(x)\) and \(\mu(x^*)\) to zero to avoid expensive posterior computations, the posterior predictive mean in \((\ref{eqn:predictivemean1})\), denoted as \(\mu_{\boldsymbol{f^*}|\mathcal{D}}\), can be calculated as 

\begin{equation}
    \mu_{\boldsymbol{f^*}|\mathcal{D}} = \Sigma_*\Sigma^{-1}\boldsymbol{f}.
\end{equation}

\subsubsection{Multi-output Gaussian Processes}
When the problem is to map an input \(x \in \mathbb{R}^P\) to an output \(y = f(x)\) using the mapping \(f : \mathbb{R}^P \rightarrow \mathbb{R}^D\), the covariance matrix should be expanded to include the covariances between \(D\) outputs. A conventional approach is to make a strong assumption of independence between different outputs and defining \(D\) different GPs, which is not valid in most cases. In this work, we use "Coregionalized Regression" approach to take into account the covariances between the outputs \cite{alvarez2011kernels}. With this approach, each output of the multi-output GP is treated as a function that is written as a linear combination of different samples from the underlying single-output GPs whose priors are constructed by the specification of the kernel. Below, we provide a summary of this approach for brevity, the details of which can be found in \cite{alvarez2011kernels}.

Let us assume that we have observations \(X_d = [x_1^d, ..., x_{m_d}^d]\) for the \(d^{th}\) function, \(f_d\), where \(m_d\) stands for the number of observations for that function, and the corresponding outputs are \(F_d = [f_d(x_1^d), ..., f_d(x_{m_d}^d)]\). The covariance matrix that encapsulates all the relations within the process can be written as
\begin{equation}
\label{eqn:hugesigma}
    \boldsymbol{\Sigma}(X,X) = 
    \begin{bmatrix}
    \Sigma(X_1, X_1) & \dots &\Sigma(X_1, X_D)\\
    \vdots & \ddots & \vdots \\
    \Sigma(X_D, X_1) & \dots & \Sigma(X_D, X_D)
\end{bmatrix}.
\end{equation}
The diagonal terms in \((\ref{eqn:hugesigma})\) modeling the auto-covariances can be constructed within the framework of single-output GP. What matters for the multi-output framework is the cross-covariance terms \(\Sigma(X_i,X_j), i \neq j\). To model these matrices, "linear model of coregionalization" (LMC) can be used. LMC is based on representing functions from underlying GPs by several independent latent functions with different kernels: Consider the set of functions \(\{f_d\}_{d=1}^D\). In the LMC framework, these functions are modeled as
\begin{equation} \label{eq30}
\begin{split}
f_d(x)= \sum_{z=1}^{\ Z}  \sum_{i=1}^{\ R_z} a_{d,z}^i u_z^i (x), 
\end{split}
\end{equation}
where \(Z\) is the number of kernels, \(u_z^i(x)\) are independent latent functions with \(cov[u_z^i(x), u_{z'}^{i'}(x)] = k_z(x, x')\) if \(i = i'\) and \(z = z'\), and zero, otherwise, \(a_{d,z}^i\) are scalar coefficients, and \(R_z\) is the number of latent functions that share the same covariance. The covariance of two functions \(k(x, x')_{d, d'} = cov[f_d(x), f_{d'}(x)]\) can be calculated as
\begin{equation}
\label{eqn:covfunc}
    k(x, x')_{d, d'} = \sum_{z=1}^Z \sum_{i=1}^{R_z} a_{d,z}^i a_{d',z}^i cov(u_z^i, u_{z'}^{i'})
\end{equation}
\begin{equation}
    = \sum_{z=1}^Z \sum_{i=1}^{R_z} a_{d,z}^i a_{d',z}^i k_z(x, x')
\end{equation}
\begin{equation}
    = \sum_{z=1}^Z b_{d,d'}^z k_z(x, x'),
\end{equation}
where \(b_{d,d'}^z = \sum_{i=1}^{R_z} a_{d,z}^i a_{d',z}^i\). The covariance matrix is then determined as
\begin{equation} \label{eq31}
k(x, x') = cov[f(x), f(x')] =  \sum_{z=1}^{\ Z} B_z k_z(x, x'),
\end{equation}
where \(B_z \in \mathbb{R}^{D \times D}\) has \(b_{d, d'}^z, d = 1, 2, ..., D, d' = 1, 2, ..., D\), as its \((d, d')^{th}\) element. Furthermore, \(B_z\) is rank \(R_z\), and called the "coregionalization matrix" measuring interrelations of different outputs. Transitioning to matrix form for multiple inputs, we can write
\begin{equation}
\label{eqn:covhugematrix}
    \boldsymbol{\Sigma}(X, X) = \sum_{z=1}^Z B_z \otimes k_z(X, X),
\end{equation}
where the symbol \(\otimes\) stands for the Kronecker product. 

For an observation set \(\mathcal{D}\) containing \(D\) different functions with \(N\) observations from each, similar to single-output GP, the posterior mean at a single point \(x^*\) is
\begin{equation}
\label{eqn:gpproof1}
    \mu_{f^*|\mathcal{D}} = \Sigma_*\Sigma^{-1} f, 
\end{equation}
where \(\Sigma \in \mathbb{R}^{ND \times ND}\) is the matrix in \((\ref{eqn:hugesigma})\), \(\Sigma_* \in \mathbb{R}^{D \times ND}\) has elements \(k(x_i, x^*)_{d, d'}\) for \(i = 1, \dots, N\), and \(d, d' = 1, \dots, D\), and \(f \in \mathbb{R}^{ND \times 1}\) is the output vector containing the outputs of \(D\) functions at \(N\) points.
\section{Method} 

\label{method}
The methodology of conventional level-\(k\) reasoning includes only the discrete levels. For instance, at a particular state, a level-0 driver might prefer the action "turn right" with a probability of 0.8, while a level-1 driver might take the same action with a probability of 0.15. This framework may not be representative enough for a driver who take the same action with 0.4 probability. To put it another way, allowing driver models to have discrete levels of reasoning, which results in action probability distributions that have jumps between these discrete levels, may provide a prohibitively limited model variety. This limitation motivates us to expand the levels from the discrete domain to a continuous domain. In this study, we model continuous levels by employing a Gaussian Process (GP), which accepts discrete levels' action probability distributions (policies), obtained through reinforcement learning (RL), at every state of a particular driver, and form policies in continuous domain. These policies belong to real-valued reasoning levels, instead of integer-valued ones, in the new framework. An illustration of the overall method is given in Figure \(\ref{fig:generalhierarchy}\), where four discrete levels from level-\(0\) to level-\(3\) are employed for representation purposes. We also introduce and prove a theorem which states that the hierarchy between the levels are also well defined for the real-valued levels when the continuous level policy is a linear combination of discrete level policies. This process is explained in the following sections. Throughout the paper, we use \((.)_l\), \(l\in \mathbb{R}^+ \cup \{0\}\), and \((.)_k\), \(k \in \mathbb{N}\), for "real-valued" and "integer-valued" concepts, respectively. 

\subsection{Continuous Hierarchical Modelling}
When we train a level-\(k\) agent using DQN, for each state \(s \in \mathcal{S}\) we obtain a set of Q-values, where each value in this set belongs to an action \(a \in \mathcal{A}\). The "Q-function" for each reasoning level-\(k\), approximated by the corresponding DQN for level-\(k\) and denoted as \(Q_k (s,a)\), is a mapping \(Q_k : S \rightarrow \mathbb{R}^A\), where \(A = |\mathcal{A}|\) is the number of available actions. We have a set \(Q = \{Q_1, ..., Q_n\}\) corresponding to \(n\) different integer-valued levels. The policy of a level-\(k\) agent, denoted as \(\pi_k(a_i|s)\), \(a_i \in \mathcal{A}\), is defined as a probability distribution over actions conditioned on the given state and it is computed by a softmax function with

\begin{equation}\label{eqSoft}
    \pi_k(a_i|s) = \frac{e^{Q_k(s,a_i)}}{\sum_{j} e^{Q_k(s,a_j)}},
\end{equation}
for each \(a_i\). Therefore, at a state \(s\), we have a set \(\pi = \{\pi_0, \pi_1, ..., \pi_n\}\) which is computed by using the set of Q-values through (\(\ref{eqSoft}\))  where \(\pi_k \in \mathbb{R}^{A}, k = 0, \dots, n\).

\begin{figure*}
    \centering
    \includegraphics[width=12cm]{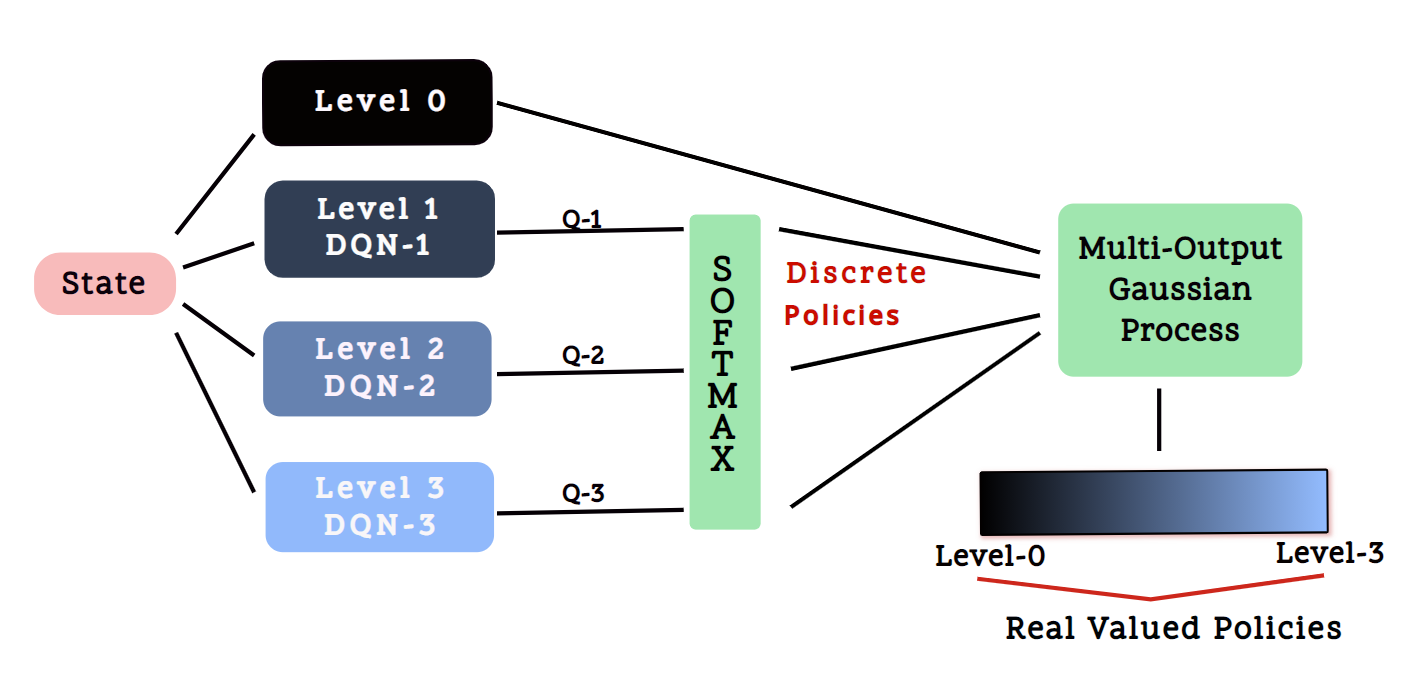}
    \caption{Creating continuous (real-valued) levels. A state is fed into DQN networks of discrete level-\(k\) reasonings. The outputs of DQNs, Q-values, are used to construct a policy through softmax function, which is then used to create a real-valued policy space for the specific state input.}
\label{fig:generalhierarchy}
\end{figure*}

\begin{remark}
Level-0 policy is generally hand crafted and not determined by an RL process. Formation of \(\pi_0\) depends on the specific structure of the level-0 policy. For example, if at a particular state \(s\), the level-0 agent deterministically chooses action \(a_j, j \in \{1,2,...,A\}\), the policy of the agent is given as \(\pi_0(a_j | s) = 1\), and \(\pi_0(a_i | s) = 0 \; \forall i \neq j\). 
\end{remark}

For each state \(s\), we define a GP by using the realizations of \(\pi\) at discrete levels 0, 1, \dots, n, denoted as \(\pi_0(\textbf{a}|s), \pi_1(\textbf{a}|s)..., \pi_n(\textbf{a}|s)\), \(\textbf{a} \in \mathbb{R}^A\). Therefore, the observation set \(\mathcal{D}\) (see Section \ref{background} for the definition of the observation set in GP) is defined as \(\mathcal{D}=\{\pi_0(\textbf{a}|s), \pi_1(\textbf{a}|s), \dots, \pi_n(\textbf{a}|s)\}\). For a given state \(s \in \mathcal{S}\), the distribution of the policy \(\pi_l : \mathcal{S} \rightarrow \mathbb{R}^A\), which is the policy for any real-valued level \(l\), can be calculated as

\begin{equation}
\label{eqn:prediction3}
    \pi_l(\textbf{a}|s) = \begin{bmatrix}
    \pi_l(a_1|s)  \\
     \vdots \\
     \pi_l(a_A|s)
\end{bmatrix} 
\sim \mathcal{N}(\Sigma_* \Sigma^{-1} \boldsymbol{f}, \Sigma_{**} - \Sigma_* \Sigma^{-1} \Sigma_*^T),
\end{equation}
through the GP framework, where \(\pi_l(\textbf{a}|s)\) is the vector containing the probabilities of all actions \(a_i, i=1,2,...,A\), at the state s, at level \(l\), and \(\boldsymbol{f} = [\pi_0(\textbf{a}|s), ... \pi_n(\textbf{a}|s)]\) is the observation vector consisting of the policies related to the integer-valued level-\(k\)s, \(k=1,2,...,n\). The covariance submatrices \(\Sigma\) and \(\Sigma_*\) are defined in \((\ref{eqn:gpproof1})\). \(\Sigma_{**}\) is defined in  \((\ref{eqn:covfunc})\) and has the elements \(k(l, l)_{a, a'}\) for \(a, a' = 1, \dots, A\).

The function with the highest probability in the GP is the mean function. In other words, let \(p(f)\) denote the probability of the function \(f \sim GP(\mu, \Sigma)\). Then,
\begin{equation}
    \argmax_f p(f) = \mu.
\end{equation}
Therefore, since a GP is a probability distribution over functions, we propose the assignment of the predictive mean (see \((\ref{eqn:gpproof1})\)) as the policy for level-\(l\). Then, at any real-valued level \(l^*\), the corresponding policy is represented by \(\widehat{\pi}_{l^*}(\textbf{a}|s) = \mu(\pi_{l^*}(\textbf{a}|s)) \in \mathbb{R}^A\). 

\begin{remark}
In our method, GP nonlinearly predicts the policy of a real-valued level which is a probability distribution over actions. Therefore, for a real-valued level \(l\), the constraint

\begin{equation}
\label{eqn:implicit}
    \sum_{i=1}^{\ A} \widehat{\pi}_l (a_i | s) =  1
\end{equation}
 should be satisfied. It is known that for linear constraints, Gaussian Process implicitly satisfies the constraint in its prediction when the inputs used in training also satisfies the linear constraint \cite{salzmann2010implicitly}. Therefore, the policy constructed by GP satisfies (\(\ref{eqn:implicit}\)). However, in some real-valued levels, even though the constraint in (\(\ref{eqn:implicit}\)) is satisfied, some of the action probabilities are predicted as negative values. Since the relative values of the probabilities are important, to handle this issue we shift the probabilities followed by a normalization to satisfy (\(\ref{eqn:implicit}\)). For instance, at a real-valued level \(l^*\) where GP predicts a negative value for one of the actions, we find
 
 \begin{equation}
     p_{shift} = \min_i \widehat{\pi}_{l^*} (a_i | s).
 \end{equation}
 Then, the probabilities are updated as
 \begin{equation}
  \label{eqn:shiftop}
     \widehat{\pi}^{new}_{l^*} (a_i | s) \leftarrow \frac{\widehat{\pi}_{l^*} (a_i | s) + |p_{shift}|}{\sum_j [\widehat{\pi}_{l^*} (a_j | s) + |p_{shift}|]}.
 \end{equation}
 Operation \((\ref{eqn:shiftop})\) defines a linear transformation on the probability value in the form
 \begin{equation}
 \label{eqn:linearform}
     \widehat{\pi}_{l^*}^{new} (a_i | s) = \beta_1  \widehat{\pi}_{l^*} (a_i | s) + \beta_2,
 \end{equation}
and since the joint gaussian distribution property of the process is invariant under linear transformations, (\(\ref{eqn:shiftop}\)) can be employed, without violating GP properties.    
\end{remark}

\subsection{Best Response in Real-Valued Reasoning Levels}
In conventional level-\(k\) approach, \(k \in \mathbb{N}\), the best response to a level-\(k\) policy is the level-\((k+1)\) policy. In this section, we extend this result by answering the following question: "What is the best response for a level-\(l\) policy, given that \(l\) is a non-negative real number?". We answer this question by first assigning utilities to the level-\(k\) players, k \(\in \mathbb{N}\), in games where their opponents are also players with integer-valued levels. Once we make the utility assignment, we provide a definition for a "pure strategy" and a "mixed strategy" in the context of level-\(k\) reasoning. Finally, we introduce and prove a best response theorem for games that are played among agents with real-valued levels for the case where the policies of real-valued levels are linear combinations of the policies of integer-valued levels.

\begin{definition}
We call the policy \(\pi_k, k \in \mathbb{N}\), as a \emph{pure strategy}, and any linear combination of pure strategies, such as \(\pi_l(a|s)=c_0\pi_0(a|s) + c_1\pi_1(a|s) +...+c_n\pi_n(a|s)\), where \(c_0 + c_1 + ... + c_n = 1\), as a \emph{mixed strategy}.
\end{definition}

\textbf{Utility assignment:} Consider a Player A with a mixed strategy \(\pi_l^A=\sum_{k=0}^n c^A_k\pi_k\), where \(l \in \mathbb{R}^+\), \(\sum_{k=0}^n c^A_k = 1\), \(c_k^A \geq 0\) \(\forall k\), and \(\pi_k\) is the pure strategy associated with level-\(k\). Similarly, consider a Player B, with a mixed strategy \(\pi_l^B=\sum_{k=0}^n c^B_k\pi_k\), \(\sum_{k=0}^n c^B_k=1\), and \(c_k^B \geq 0\) \(\forall k\). The utility of Player A, playing against Player B can be calculated as
\begin{equation}
\label{eqn:utility1}
    u_A(\pi_l^A,\pi_l^B) = \sum_{k=0}^n c_k^A u_A(\pi_k,\pi_l^B),
\end{equation}
where \(u_A(\pi_k,\pi_l^B)\) is the utility induced when the pure strategy \(\pi_k\) is played against \(\pi_l^B\). This utility can be expanded as
\begin{equation}
\label{eqn:utility2}
    u_A(\pi_k,\pi_l^B) = \sum_{j=0}^n c^B_j u_A(\pi_k,\pi_j).
\end{equation}
Substituting \((\ref{eqn:utility2})\) into \((\ref{eqn:utility1})\), it is obtained that
\begin{equation}
\label{eqn:expectedutility}
    u_A(\pi_l^A,\pi_l^B) =\sum_{k = 0}^n \sum_{j = 0}^n c_k^A c^B_j u_A(\pi_k,\pi_j).
\end{equation}

In conventional level-\(k\) reasoning approach, \(k \in \mathbb{N}\), the best response to a level-\(k\) policy is the level-\((k+1)\) policy. On the other hand, the best response in other situations, where the difference between the levels are different than \(1\), is not defined. For example, the situation where a game is played between a level-0 agent and a level-2 agent is not defined, in terms of the payoffs of the agents. Using this best-response definition of the level-\(k\) approach, we assign a \(+1\) utility to a level-\((k+1)\) agent if the agent plays against a level-\(k\) agent, and a 0 utility, otherwise.

\begin{theorem}
Consider a level-\(l\) policy, \(\pi_l\), where \(l \in \mathbb{R}^+\), and a set of level-\(k\) policies, \(\pi_k\), \(k=0,1,..., n\). When the real-valued policy is a linear combination of integer-valued policies in the form \(\pi_l (a|s) = c_0\pi_0(a|s) + c_1\pi_1(a|s) + ... + c_{n-1}\pi_{n-1}(a|s)\), where \(\sum_{j = 0}^{n-1} c_j = 1\) and \(0 \leq c_j \leq 1\) \(\forall j \in \{0, \dots, n-1\}\), then the policy \(\pi_{\omega} (a|s)\) that is a best-response (BR) to \(\pi_l\) is given as

\begin{equation}
\label{eqn:theorem1}
\pi_{\omega} (a|s) = BR(\pi_{l} (a|s))= \sum_{m \in \mathcal{M}} \gamma_m \pi_m(a|s),
\end{equation}
where \(\mathcal{M} = \{i + 1 | c_i \geq max\{c_0, c_1, \dots, c_{n-1}\}\}\), \(i = 0, 1, \dots , n - 1\) and \(\gamma_m\) can be any value subject to the constraints \(0 \leq \gamma_m \leq 1\) and \(\sum_{m \in \mathcal{M}} \gamma_m = 1\). 
\end{theorem}

\begin{proof}
Without loss of generality, suppose that \(c_t = c_{t+1} = ... = c_h = c^*\), \(0 \leq t \leq h \leq n-1\), and \(c^* \geq max\{c_0, c_1, \dots, c_{n-1}\}\). Consider a policy \(\pi = \gamma_1 \pi_1 + \gamma_2 \pi_2 + \dots + \gamma_n \pi_n\), where \(\sum_i \gamma_i = 1\), and \(\gamma_i \geq 0\) \(\forall i\in\{1, \dots, n\}\). The expected utility of \(\pi\) playing against \(\pi_l\) can be calculated as
\begin{equation}
    u(\pi,\pi_l) = \gamma_1 c_0 + ... + \gamma_{t+1} c_t + \gamma_{t+2} c_{t+1} + ... + \gamma_{h + 1} c_{h} + ... + \gamma_n c_{n-1}
\end{equation}
\begin{equation}
    = \gamma_1 c_0 + ... + \gamma_{t+1} c^* + \gamma_{t+2} c^* + ... + \gamma_{h + 1} c^* + ... + \gamma_n c_{n-1}
\end{equation}
\begin{equation}
\label{eqn:mergedgamma}
    = \gamma_1 c_0 + ... + (\gamma_{t+1} + \gamma_{t+2} + ... + \gamma_{h + 1}) c^* + ... + \gamma_n c_{n-1}.
\end{equation}
Let 
\begin{equation}
\label{eqn:sumgammas}
    \overline{\gamma} = \gamma_{t+1} + \gamma_{t+2} + ... + \gamma_{h + 1}.
\end{equation}
Then, using (\(\ref{eqn:sumgammas}\)), we can write (\ref{eqn:mergedgamma}) as
\begin{equation}
    u(\pi,\pi_l) = \gamma_1 c_0 + ... + \overline{\gamma} c^* + ... + \gamma_n c_{n-1}
\end{equation}
\begin{equation}
    \leq \gamma_1 c^* + ... + \overline{\gamma} c^* + ... + \gamma_n c^*
\end{equation}
\begin{equation}
    = (\gamma_1 + ... + \overline{\gamma} + ... + \gamma_n)c^*
\end{equation}
\begin{equation}
    = c^*.
\end{equation}
Therefore, \(c^*\) is the maximum utility that can be obtained against \(\pi_l\). Defining \(\mathcal{J} = \{0, 1, \dots, n\}\) and \(\mathcal{M} = \{t+1, t+2, \dots, h, h+1\} = \{i+1|c_i \geq max\{c_0, c_1, \dots, c_{n-1}\}\}\), if we choose \(\overline{\gamma} = 1\) (see (\(\ref{eqn:sumgammas}\))), which means that \(\gamma_j = 0\), for \(j \in \mathcal{J} \backslash \mathcal{M}\), and if we choose \(\gamma_m\), \(m \in \mathcal{M}\) such that \(\sum_{m \in \mathcal{M}} \gamma_m = 1\) and \(\gamma_m \geq 0\), \(\forall m \in \mathcal{M}\), then the utility of the policy \(\pi_{\omega} = \sum_{m \in \mathcal{M}} \gamma_m \pi_{m}\) against \(\pi_l\) is \(u(\pi, \pi_l) = c^*\) which is the maximum possible utility. 
\end{proof}

\section{Validation with Traffic Data}
\label{validation}
In this section, we use US101 data \cite{us101} to show how representative the real-valued reasoning levels are in terms of modeling driver behavior. To obtain the continuous, real-valued driver reasoning levels, we first need the conventional level-\(k\) models. Observation and action spaces, the reward function, and synergistic employment of deep Q-learning and game theory employed to obtain these discrete models for this dataset is already explained in \cite{albaba2020driver}. Therefore, for brevity, we omit this part. We employ these discrete levels as inputs to the Gaussian Process (GP) to obtain real-valued continuous levels, using the process given in Section \ref{method}. We use Kolmogorov-Smirnov (K-S) goodness of fit test \cite{conover1972kolmogorov}, to compare the policies of the proposed model and that of real human drivers.

\subsection{Construction of the Gaussian Process}
In constructing the multi-output GP in order to create the policy space, we chose Matérn kernel due to its finite differentiability, which makes it a more reasonable choice for most dynamical systems \cite{stein1999interpolation}. The Matérn kernel we used is given as 
\begin{equation} \label{eq19}
\begin{split}
k(x,x') = \sigma^2(1+\frac{\sqrt{3}(x-x')}{\beta})exp(-\frac{\sqrt{3}(x-x')}{\beta})
\end{split}
\end{equation}
where \(\beta \in \mathbb{R}^+\) and \(\sigma \in \mathbb{R}^+\) are the length scale and variance parameters, respectively. For optimization of the parameters and constructing the GP, we use GPy library \cite{gpy2014}. The method presented in this paper uses the LMC model \cite{alvarez2011kernels} with \(R_Z=7\) and \(Z=7\), where \(Z\) and \(R_Z\) are the number of kernels and the number of latent functions sharing the same covariance, respectively. Within these \(7\) kernels, one is a "bias kernel" that adjusts the mean of the function put into the GPy, whereas the other six kernels are Matérn kernels with different \(\beta\) constraints. We constrained \(\beta\) in six Matérn kernels to the values in the set \{0.25,0.5,0.75,1.0,1.25,1.5\} in order to increase the contributions coming from the neighboring levels. 

\subsection{Comparison of Continuous Game Theoretical Policies and Policies Extracted from Driver Data}
We first provide definitions that we will be using extensively throughout the following sections. Then, we will detail the approach used in the comparison phase. 

\begin{definition}{}
Discrete (integer-valued) game theoretical (DGT) policy is a probability distribution over actions. We use DGTs for comparison purposes. To see how DGTs are obtained, see \cite{albaba2020driver}. In this paper, the discrete level policies are level-0, level-1, level-2 and level-3 policies. It is noted that these policies use a discrete observation space and a continuous action space sampled from discrete bins.
\end{definition}

\begin{definition}{}
Continuous (real-valued) game theoretical (CGT) policy, the one proposed in this paper, is also a probability distribution over actions. However, the policies can be of a reasoning level-\(l\), where \(l\) is a real number in the interval \([0, 3]\). 
\end{definition}

\begin{algorithm}
	\caption{Comparison of CGT Policies with Traffic Data Policies for a Single Driver} 
	\begin{algorithmic}[1]
        \For{$i = 1$ to totalStates}
        
            \If{\( n_{visits}^i \geq n_{th}\)}
                \State Increment \(n_{comparisons}\) by 1
                \State Get Q-values of discrete levels \{\(Q_1, Q_2, Q_3\)\} by \par
                     \hskip\algorithmicindent using corresponding DQN networks
                \State Get Q-values \(Q_0\) of level-0 driver (see Section \ref{method})
                \State Construct a multioutput GP model where x-axis \par
                     \hskip\algorithmicindent contain the interval \( [0,3]\) and y-axis carries\par
                     \hskip\algorithmicindent policies \{\(\pi_0, \pi_1, \pi_2, \pi_3\)\}.
                \State Set initial levels \emph{InitLevels} \par
                     \hskip\algorithmicindent to \{\(0, 1, 2, 3\)\}
                \State Set \(MaxSteps\) to 50
                
                \For{j = 1 to 4}
                    \State Run Algorithm \ref{alg:simanneal} with initial level \(InitLevels_j\) \par
                        \hskip\algorithmicindent \hskip\algorithmicindent and \(MaxSteps\) to find \(l_j\) and \(criticalvalue_j\).
                \EndFor
                
                \State Set the optimum level \par
                     \hskip\algorithmicindent \(l_{opt} = \argmax_{l_j} criticalvalue\)
                \State Set the optimum critical value \par
                     \hskip\algorithmicindent \(crit_{opt} = \max criticalvalue\) 
                \If{\(crit_{opt} \leq \theta_{th}\)}
                    \State Do no increment \(n_{success}\)
                \Else
                    \State Increment \(n_{success}\) by 1
                \EndIf
            \EndIf
                
        \EndFor
        \State Percentage of successfully modeled states = \( 100 \frac{n_{success}}{n_{comparisons}}\)
	\end{algorithmic} 
	\label{alg:comparison}
\end{algorithm}

\begin{algorithm}
	\caption{Simulated Annealing for Finding the Level of a Driver} 
	\begin{algorithmic}[1]
	    \State Set initial temperature \(T_S\) to 2.
	    \State Set initial level to \(l\).
	    \State Find the action probability distribution \(P_{data}\) at level \(l\) from traffic data
	    \State Find the action probability distribution \(P_{model}\) at level \(l\) with GP
	    \State Shift and renormalize \(P_{model}\) if it contains a negative value (see Section \ref{method})
	    \State Calculate initial cost \(g(l)\) by using \(P_{model}\) and \(P_{data}\), store the critical value in \(cv\)
		\For {$iteration=1,2,\ldots,MaxSteps$}
			\State Find a neighbor \(l_{new}\) around \(l\)
			\State Update the action probability distribution \(P_{model}\) at level \(l_{new}\) with GP
			\State Calculate \(g(l_{new})\) by using \(P_{model}\) and \(P_{data}\), update \(cv_{new}\).
			\State Calculate acceptance probability \(p_{acc}\) using \(exp(-\delta / T_S)\), where \(\delta = cv_{new} - cv\)
			\State Set random variable \(p_{random}\) from the interval \([0,1]\) uniformly
			\If{$p_{acc} > p_{random}$}
			    \State Assign \(l_{new}\) to \(l\)
			    \State Assign \(cv_{new}\) to \(cv\)
            \EndIf
            \State Update \(T_S\) with \(T_S = 0.90  T_S\)
		\EndFor
		\State Return \(l\) and \(cv\) which are the final results of the search
	\end{algorithmic} 
	\label{alg:simanneal}
\end{algorithm}

Real driver policies, the probabilities over actions, are found by calculating the frequencies of the actions chosen by the driver at a given state. The probabilities below \(0.01\) are fixed to \(0.01\) in order to eliminate the possibility of zero-probability actions. Once the driver policy is obtained from data, this policy is compared with the continuous policies created by the GP using the K-S test. During the comparison, the search over the continuous policy space is conducted using simulated annealing \cite{kirkpatrick1983optimization} (see Algorithm \ref{alg:simanneal}). The algorithm used for a single driver is given in Algorithm \ref{alg:comparison}. In Algorithm \ref{alg:comparison}, \(l_j\) corresponds to the level with the highest critical value found by the \(j\)th simulated annealing started at \(InitLevels_j\). \(l_{opt}\) corresponds to the level of the driver with the highest critical value among all the searches and \(crit_{opt}\) is the result of K-S test at that level measuring the similarity of CGT policy and traffic data policies.

\section{Results and Discussion}
\label{results}
Figure \(\ref{fig:individuals}\) shows the percentage of successfully modeled real driver states using Discrete Game Theoretical (DGT) and Continuous Game Theoretical (CGT) policies. The x-axes correspond to the driver ID numbers whereas the y-axes display the percentage of correctly modeled states. Therefore, each vertical line corresponds to the successfully modeled state percentage for an individual driver. The figure shows that CGT can model a significantly higher percentage of states, as expected. This is not a surprising result since CGT contains an infinitely large policy space while DGT contains only discrete levels. This figure is provided to demonstrate the flexibility obtained by employing real-valued reasoning levels instead of integer-valued ones. 

Another demonstration of this is given in Figure \(\ref{fig:gridresult}\). The x and y axes correspond to percentage of successful models made by different methods and the color of a cell in the grid represents the number of drivers falling into that cell. For instance, the white cells in Figure \(\ref{fig:gridresult}\) indicate that there are approximately 80 drivers whose states are modeled by DGT policies with nearly \(45\%\) success and by CGT policies with nearly \(85\%\) success. In this representation, being in the upper side of x = y line is interpreted as an indicator of better modeling performance. In other words, the brighter the cells at the upper left part, the better the model in y-axis is than the model in x-axis is.

\begin{figure}
    \centering
    \includegraphics[width=14cm]{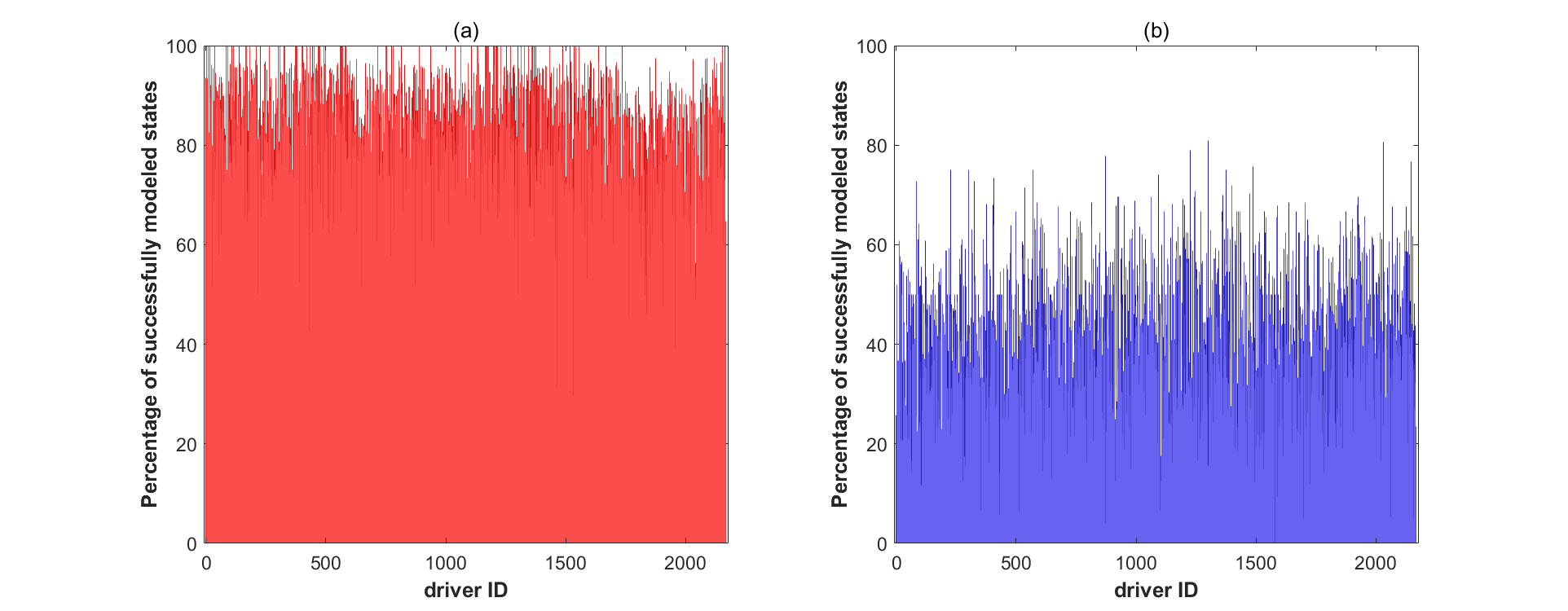}
    \caption{Percentage of successfully modeled states, for each individual driver, by (a) CGT and (b) DGT policies.}
    \label{fig:individuals}
\end{figure}

\begin{figure}
    \centering
    \includegraphics[width=9cm]{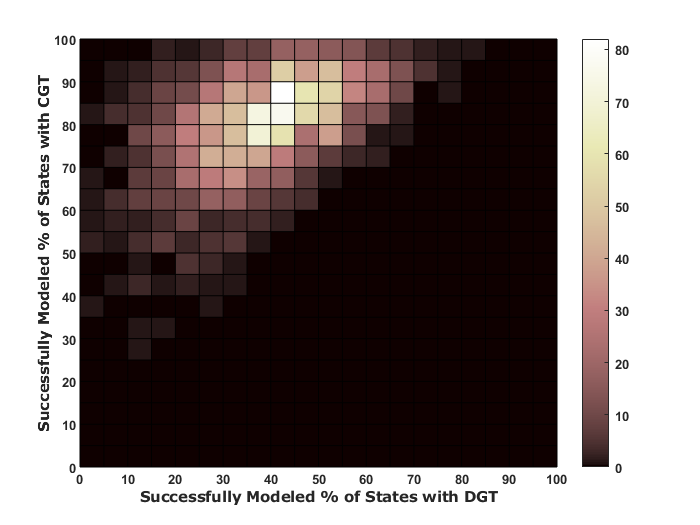}
    \caption{A 2D color map that displays the number of drivers whose x\% of the
visited states are successfully modeled by the DGT policy and y\% of  the visited states are successfully modeled by
by the CGT policy.}
    \label{fig:gridresult}
\end{figure}

\begin{table}[]
\centering
\begin{tabular}{|l|c|}
\hline
Method                     & \multicolumn{1}{l|}{Mean \% of Successfully Modeled States} \\ \hline
Discrete State Space DGT   & 34.75                                                       \\ \hline
Continuous State Space DGT & 72.72                                                       \\ \hline
Discrete State Space CGT   & 80.30                                                       \\ \hline
\end{tabular}
\captionof{table}{Comparison of DGT and CGT in terms of modeling performance\label{tab:successtable}}
\end{table}

It is noted that CGT or DGT can use either continuous or discrete state spaces as their observation space. This is independent of the policies being discrete or continuous. The results in Figures \(\ref{fig:individuals}\) and \(\ref{fig:gridresult}\) are obtained using a discrete state space. In \cite{albaba2020driver}, it is shown that an approach relying on continuous state space modeling is better in terms of modeling performance. This is also demonstrated in Table \(\ref{tab:successtable}\), where continuous state space DGT is shown to model a considerably higher percentage of states than the discrete state space DGT. However, it is also shown in Table \(\ref{tab:successtable}\) that even with a discrete state space, CGT performs better than DGT.

What is more interesting to observe is the distribution of levels among real drivers. Figure \(\ref{fig:scatter}\) shows how this distribution appears for each driver in the traffic data. Each yellow dot on the figure indicates a certain level (given in the y-axis) assigned for a given state of an individual driver (driver ID is given in the x-axis). It is observed that most of the drivers' reasoning levels are accumulated between level-\(0\) and level-\(1.75\). The significant gap at level-\(1.75\) and above indicates that in general human drivers act as if other drivers in the traffic do not have reasoning levels of 1 or above. Curiously, this finding is in harmony with the experimental results reported in \cite{camerer2011behavioral}, \cite{nagel1995unraveling} and pointed out in \cite{lee2012game}, where it is found that in the Beauty Contest game, for example, with \(\rho = 2/3\) or \(\rho = 1/2\), humans behave mainly as level 1 and 2 and rarely level-3. In \cite{nagel1995unraveling}, unlike in this work, a simple level-0 strategy, corresponding to the bet of 50, which is the average number, is used. On the other hand, in this work, level-0 reasoning corresponds to a strategy which is a driving style determining actions in many different scenarios, and consequently, it is a dramatically more complex strategy in comparison to \cite{nagel1995unraveling}. This explains the low frequency of level-0 players in \cite{nagel1995unraveling}.

\begin{figure}
    \centering
    \includegraphics[width=15cm]{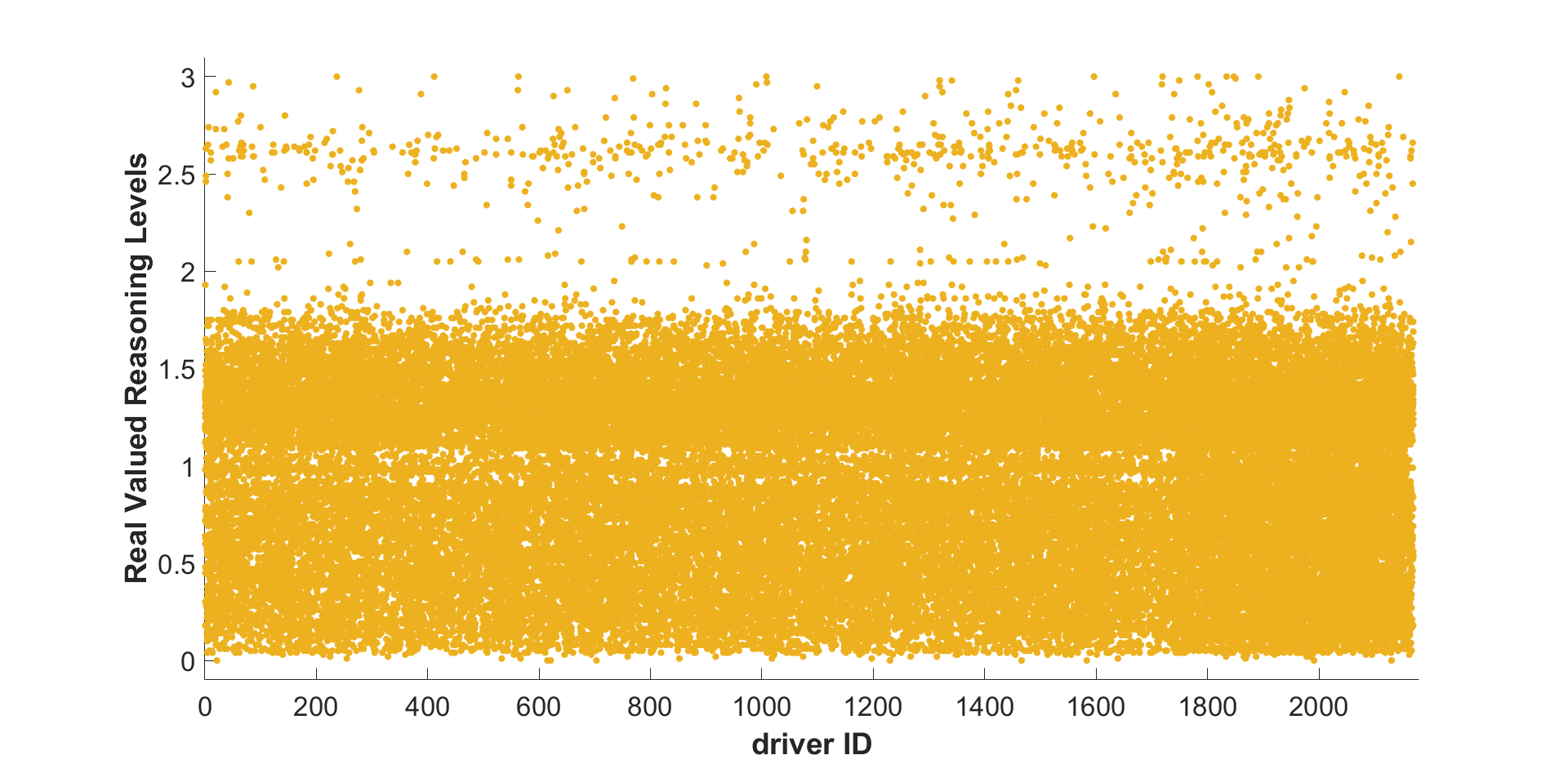}
    \caption{For each driver, detected levels for all states that are correctly modeled are shown with markers. The states where the null hypothesis is rejected are excluded.}
    \label{fig:scatter}
\end{figure}

\begin{figure}
    \centering
    \includegraphics[width=15cm]{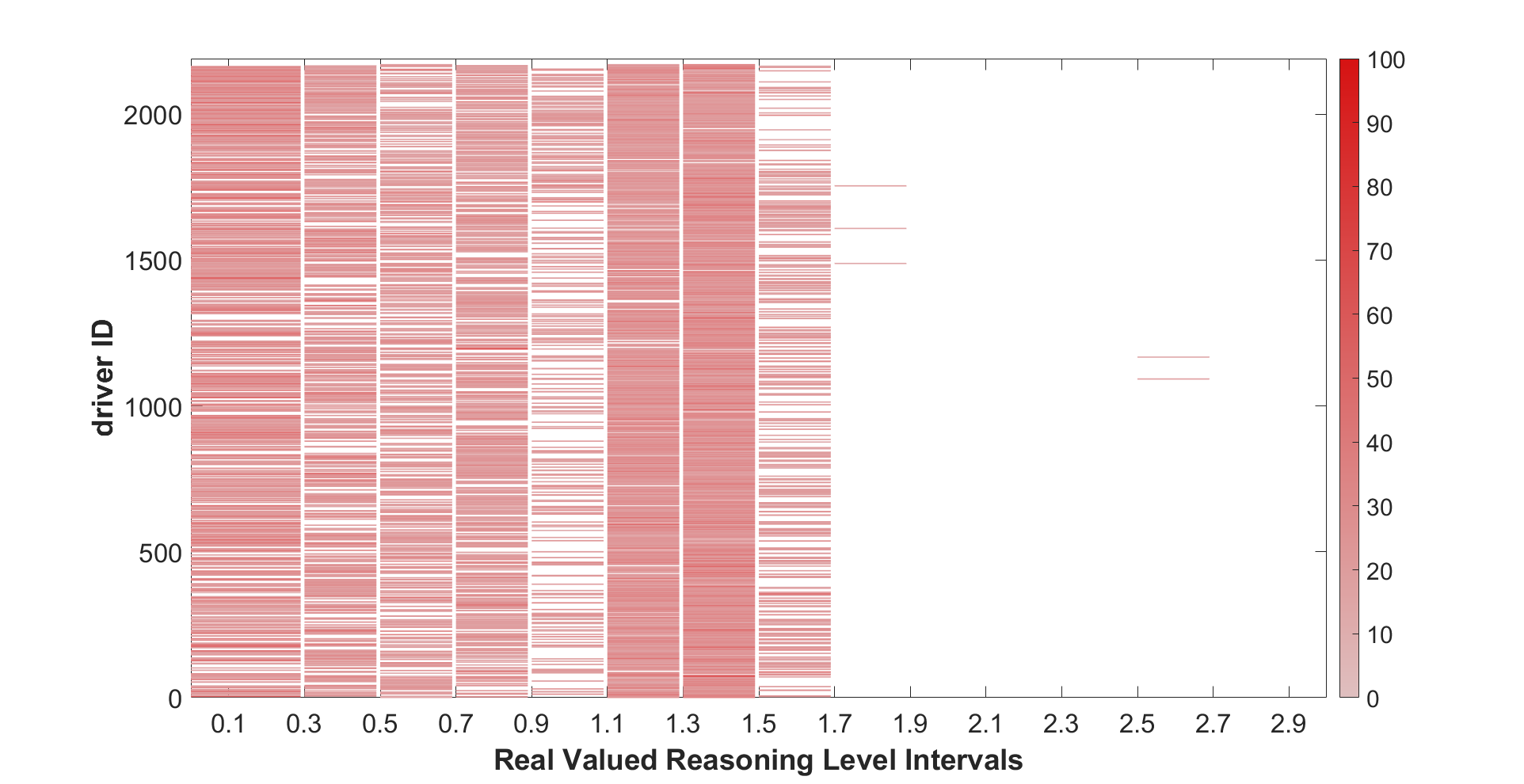}
    \caption{For each driver, three of the most occupied reasoning intervals are displayed. The color-bar indicates the percentages of the states that belong to a given interval.}
    \label{fig:colorplot}
\end{figure}

Another demonstration of level distribution of real drivers is given in Figure \(\ref{fig:colorplot}\). The y-axis shows the driver ID whereas the x-axis indicates the intervals of real-valued reasoning levels. In particular, intervals are quantized uniformly in which each interval has a size of 0.2 such as \([0.9,1.1)\) except the first and last interval, which are given as \([0,0.3)\) and \([2.7,3.0]\). The color-bar on the right shows the percentage of the states in the corresponding interval. The figure demonstrates that human drivers mostly have reasoning levels in the interval \([1.1,1.5)\), followed by the interval \([0.0, 0.9)\). As suggested by the previous result, reasonings in higher levels are not common. 

\section{Summary}
\label{conclusion}
In this study, we proposed a refinement to level-\(k\) game theoretical reasoning concept where we introduced continuous levels through multi-output Gaussian Processes. This fills a gap in the literature of level-\(k\) game theoretical reasoning concept, which limits us to a few discrete levels for modeling human driver behavior. We also put forward a best response theorem to show that the hierarchical relation is still well defined when the real-valued level policies are linear combinations of integer-valued ones. The proposed approach is validated on real traffic data. Employing US101 driving data, and using Kolmogorov-Smirnov goodness of fit test, it is shown that most of the drivers embrace a behavioral reasoning that is consistent with previous studies.


\bibliographystyle{ieeetr}
\bibliography{references}  

\end{document}